\documentclass{llncs}

\usepackage{amsmath}
\usepackage{fca}
\usepackage{tikz}
\usetikzlibrary{arrows,positioning}
\usepackage{pgfplots}
\usepackage{listings}
\usepackage[algo2e, vlined, linesnumbered]{algorithm2e}
\usepackage{bbm}
\usepackage{float}
\usepackage{caption}
\usepackage{subcaption}
\usepackage{array}
\usepackage[bottom]{footmisc}
\usepackage{footnote}

\usepackage{marginnote}

\newcommand{\textoverline}[1]{$\overline{\mbox{#1}}$}

\title{Interactive Error Correction in Implicative Theories}
\author{Sergei O. Kuznetsov\inst{1} \and Artem Revenko\inst{1}\inst{2}}
\date{}
\institute{National Research University Higher School of Economics\\
         Pokrovskiy bd. 11, 109028 Moscow, Russia\\
         \and
         Technische Universit\"{a}t Dresden\\
         Zellescher Weg 12-14, 01069 Dresden, Germany\\
         \email{skuznetsov@hse.ru}, \\
         \email{artem\_viktorovich.revenko@mailbox.tu-dresden.de}}
 
\begin{document}
 
\maketitle

\begin{abstract}
Errors in implicative theories coming from binary data are studied. First, two classes of errors that may affect implicative theories are singled out. Two approaches for finding errors of these classes  are proposed, both of them based on methods of Formal Concept Analysis. The first approach uses the cardinality minimal (canonical or Duquenne-Guigues) implication base. The construction of such a base is computationally intractable. Using an alternative approach one checks possible errors on the fly in polynomial time via computing closures of subsets of attributes. Both approaches are interactive, based on questions about the validity of certain implications.  Results of computer experiments are presented and discussed.
 
\end{abstract}

\section{Introduction}
\subsubsection{Motivation}
 
Implicative theories consisting of formulas of the form ``if $A$, then $B$'' provide a standard way for describing the structure of domain knowledge. They are extensively used in various research areas, e.g., biology~\cite{kknd11}, pharmacology~\cite{cla13,Blinova2003}, semantic web~\cite{kltlkl12}, etc.
It is difficult to overestimate their importance for knowledge discovery \cite{frawley1992knowledge,valtchev2004formal}, decision making \cite{rasmussen1985role}, classification \cite{mirkin1996mathematical}, ontology engineering \cite{baader2010description}. In many cases the exactness of rules plays a crucial role, for example in research related to strictly formalized domains like Boolean algebras \cite{Kwuida2006}, algebraic lattices \cite{Dau2000}, or algebraic identities \cite{revenko2014automatized}.
 
In many applications an exact implicative theory is constructed from a piece of available data. It is well-known that a single mistake in this data can drastically change the resulting implicative theory \cite{Ganter1999} (the same is true for association rules if there are some exceptions and an error). The implicative theory is not going to recover from this error even if further error-free data is added to the underlying set. Therefore, implicative theories are not error tolerant. However, in the real-world applications, especially if multiple users are expected to work with data, it is hardly imaginable to guarantee the absence of errors. More than that, someone may be willing to spoil the result on purpose. Therefore, a procedure for recovering from errors is essential for the usage of implicative theories.
 
Here we assume that in the beginning there is already some data on hands and new data arrives in the work flow. The goal is to guarantee the correctness of the implicative theory with respect to the initial data which are considered to be reliable. We do not assume that a user, which is going to work with the data and the implicative theory, is always able to explicitly state any knowledge about data domain or has any knowledge about methods in use. That is why it is important to develop a transparent and easy  method for error correction. In particular, it is important to find and output possible errors in a human understandable form. To attain this goal a natural framework can be that of Formal Concept Analysis (FCA) \cite{Ganter1999}, where methods and algorithms for finding implicative theories of binary data (formal contexts) are well elaborated and widely used \cite{Ganter1984,RyDiB11}.

\subsubsection{Related Work}
 
Methods for imputing missing values are well studied. In \cite{Song2007a} and \cite{Silva-Ramirez2011a} detailed overviews of existing techniques are presented. Among others there are techniques of ignoring entries with missing values, imputing average values, and more complicated ones such as decision trees, neural networks \cite{Silva-Ramirez2011a}, Nearest Neighbours approach \cite{Jain2011a}. Having a missing value, there is no need to search for an error, as it is clear from the problem statement which value should be changed (or imputed). An approach proposed in this paper bares some similarity to the Nearest Neighbor method, but  aims at solving a different task. Besides that, the imputation techniques (like, e.g. averaging) are mostly not relevant for binary data.
 
Error finding and eliminating are widely discussed in various fields of computer science. The problems of lineage or data provenance, where one needs to explain errors, trace reasons for a query, etc. are well-known in KDD domain~\cite{spg}. These techniques are very useful and efficient, however, they are not appropriate for correcting errors in binary data tables.
 
In \cite{Fan2010a} an impressive way of using expert knowledge presented in the form of editing rules and certain regions for databases are surveyed. Information in the form of editing rules prevents the errors from getting in to the database. The approach presented in this paper aims at finding and correcting errors without any previously formalized knowledge.
 
The paper \cite{bs09} presents an interesting approach to dealing with mistakes in answering questions (like the ones we will discuss below) in the process of knowledge base completion within the framework of Description Logics.  This approach allows recovering from such mistakes in such an effective manner that the information input is used upon mistake recovery. However, the detection and correction of mistakes is left to pinpointing.
 
Pinpointing is a very successful technique for recovering from inconsistencies.
The goal of pinpointing is the following: for a given inconsistent set of rules (not only implicative) find maximal consistent subsets \cite{baader2007pinpointing,meyer2006finding}. This technique is successfully applied in different description logics. The complexity of pinpointing is normally beyond polynomial. An approach introduced in this paper (Section \ref{find}, base approach) is closely related to pinpointing; it proceeds from knowledge base constructed from data. The complexity is also beyond polynomial. However, an alternative approach (Section \ref{find}, closure approach) takes the advantage of having the data and proposes a polynomial-time solution. In this work we do not modify the knowledge base directly, but we correct the errors in data in such a way that the corresponding implicative theory becomes error-free.
 
As implicative theories is another view of Horn theories \cite{Ganter1999}, the problem of finding explanations in Horn theories turns out to be closely connected to our problem. Namely, an entry in the binary data table can also be considered as a fact to be explained. In \cite{Kautz1993a} it is shown that such explanations may be found in polynomial time. However, here we aim at explaining existence or absence of all attributes at the same time. Also we state our task and our solutions in a different language and provide algorithms for practical usage. The case of negative attributes is not covered in \cite{Kautz1993a} as opposed to this work.
 
The present paper is a follow-up work to \cite{Revenko2012c}.
 
\begin{remark}
In this paper we assume that working with a data domain we can ask an expert in the domain whose answers are always correct. However, we should ask as few questions as possible.
\end{remark}
\begin{remark}
All sets and contexts we consider in this paper are assumed to be finite.
\end{remark}
 
\subsubsection{Contributions}
We introduce an interactive procedure for making implicative theory of data error-free. This goal is achieved via finding and eliminating errors in rows of data tables. In FCA terms, we propose an approach for finding errors in descriptions  of new objects (intents in terms of FCA) that affect the canonical implication base.
\begin{enumerate}
  \item We introduce two possible classes of errors in data tables (formal contexts, Section \ref{err_classes}).
  \item We introduce two approaches to finding and eliminating errors of certain classes (Sections \ref{find}). Both aim at restoring dependencies from data domain and eliminating errors in implicative theory.
  \begin{enumerate}
    \item One approach is based on finding those implications from an implication base that are not respected by the new object intent (base approach). However, the base approach leads to an intractable solution, because constructing an implication base is intractable.
    \item We introduce another approach (closure approach), where we do not need to compute the set of all implications, and prove its effectiveness. We show that it helps to find all possible errors of certain types (Proposition \ref{proposition1}) in polynomial time (Proposition \ref{proposition2}).
  \end{enumerate}
  The approaches are experimentally compared  in Section \ref{experiment}.
\end{enumerate}
 
\section{Main Definitions}
 
In what follows we keep to standard definitions of FCA~\cite{Ganter1999}. Let $G$ and $M$ be sets and let $I\subseteq G\times M$ be a binary relation between $G$ and $M$. The triple $\context := \GMI$ is called a \emph{(formal) context}. Set $G$ is called a set of \emph{objects}, set $M$ is called a set of \emph{attributes}, $I$ is called \emph{incidence relation}.
 
Consider mappings $\varphi\colon 2^G\to 2^M$ and $\psi\colon 2^M\to 2^G$: $\varphi(X) := \{m\in M\mid gIm \  \mbox{ for all\ } g\in X\},\ \psi(A) := \{g\in G\mid gIm \ \mbox{ for all\ } m\in A\}.$ Mappings $\varphi$ and $\psi$ define a \emph{Galois connection} between $(2^G,\subseteq)$ and $(2^M,\subseteq)$,  i.e. $\varphi(X) \subseteq A \Leftrightarrow \psi(A)\subseteq X$. Hence, for any $X_1, X_2\subseteq G$, $A_1, A_2\subseteq M$ one has
\begin{enumerate}
    \item $X_1\subseteq X_2 \Rightarrow \varphi(X_2) \subseteq \varphi(X_1)$ \label{prop1}
    \item $A_1\subseteq A_2 \Rightarrow \psi(A_2) \subseteq \psi(A_1)$ \label{prop2}
    \item $X_1\subseteq \psi\varphi (X_1) \mbox{ and } A_1\subseteq \varphi\psi(A_1)$ \label{monoton}
\end{enumerate}
Usually, instead of $\varphi$ and $\psi$ a single notation $(\cdot)^{\prime}$ is used. $(\cdot)'$ is usually called a \emph{derivation operator}. For $X \subseteq G$ the set $X^{\prime}$ is called the \emph{intent} of $X$. Similarly, for $A \subseteq M$ the set $A^{\prime}$ is called the \emph{extent} of $A$.
Operator $(\cdot)''$ is idempotent, extensive and monotone, i.e., has the properties of algebraical closure both on $2^G$ and $2^M$. Hence, $(Z)^{\prime\prime}$ is called \emph{closure} of $Z$ in $\context$ for $Z \subseteq M$ or $Z \subseteq G$. If $(Z)^{\prime\prime} = Z$, set $Z$ is called  \emph{closed} in $\context$. Applying Properties \ref{prop1} and \ref{prop2} consequently one gets the \emph{monotonicity} property: for any $Z_1, Z_2 \subseteq G$ or $Z_1, Z_2 \subseteq M$ one has $Z_1 \subseteq Z_2 \Rightarrow Z_1'' \subseteq Z_2''$.

In \cite{rodriguezgeneralized} authors introduce a generalized framework for considering positive and negative attributes. In this paper we also introduce negative attributes, however, we do not need the whole framework for our purpose. Our definitions comply with the definitions from \cite{rodriguezgeneralized}.

The set $\overline{M} := \{\overline{m}\ |\ m \in M\}$ is called the set of \emph{negative} attributes. Consider the following relation $\overline{I} := \{(g, \overline{m})\ |\ (g,m) \in (G \times M) \setminus I\}$ between $G$ and $M$. The context $\context^{\delta} := (G, M \cup \overline{M}, I \cup \overline{I})$ is called the \emph{dichotomized context} to $\context$, the corresponding derivation operator is denoted by $(\cdot)^{\delta}$. Let $X \subseteq G$. Note that $\overline{m} \in X^\delta$ iff $m \notin g'$ for all $g \in X$. If $\overline{m} \in X^\delta$ then, as it does not lead to ambiguity, we informally write $\overline{m} \in X'$. In this paper objects and context are represented without negative attributes, however, in the processing stage they are normally converted to the dichotomized representation in order to be able to work with negative attributes.

Consider the context $\context^{\overline{\delta}} = (G, \overline{M} \cup \overline{\overline{M}}, \overline{I} \cup \overline{\overline{I}})$. This context is isomorphic to the context $\context^{\delta} = (G, M \cup \overline{M}, I \cup \overline{I})$ and $\overline{\overline{m}} \in X^{\overline{\delta}}\ \Leftrightarrow\ m \in X^{\delta}$.
 
A \emph{formal concept} of a formal context $\GMI$ is a pair $(X, A)$, where $X \subseteq G,\ A\subseteq M,\ X'=A,$ and $A'=X$. The set $X$ is called the extent, and the set $A$ is called the intent of the concept $(X, A)$.
 
One says that an object $g$ such that $g'\neq \emptyset$ is \emph{reducible} in a context $\GMI$ iff $\exists X \subseteq G\setminus \{g\}:\ g' = \bigcap \limits_{j\in X} j'$. Removing reducible objects does not change the concept lattice up to isomorphism.
 
In this paper implicative theories are formalized in terms of implication bases.
An \emph{implication} of $\context := \GMI$ is defined as a pair $(A,B)$, written $A\to B$, where $A, B\subseteq M$. $A$ is called the \emph{premise}, $B$ is called the \emph{conclusion} of implication $A\to B$. Implication $A\to B$ is \emph{respected by a set of attributes} $N$ if $A \nsubseteq N$ or $B \subseteq N$. Implication $A\to B$ holds (is valid) in $\context$ if it is respected by all $g'$, $g\in G$, i.e. every object, that has all the attributes from $A$, also has all the attributes from $B$, or, equivalently, if $A' \subseteq B'$. Implications satisfy \emph{Armstrong rules}:
$${{}\over{A\to A}}\quad , \quad
        {{A\to B}\over{A\cup C\to B}}\quad , \quad
        {{A\to B, B\cup C\to D}\over{A\cup C\to D}}$$
\emph{Support} of implication $A\to B$ in context $\context$ is $(A\cup B)'$, i.e., the set of all objects of $\context$, whose intents contain the premise and the conclusion of the implication. A \emph{unit implication} is defined as an implication with only one attribute in conclusion, i.e. $A \to b$, where $A \subseteq M,\ b\in M$. Using Armstrong rules, every implication $A \to B$ can be represented as a set of \emph{unit  implications} $\{A \to b\ |\ b \in B\}$, so  one can always observe only unit implications without loss of generality.
 

Consider implications of the form $A \to \overline{b}$, where $A \subseteq M, \overline{b} \in \overline{M}$ in the dichotomized context $\context^{\delta}$. This implication is said to be respected by $N \subseteq M$ if $A \nsubseteq N$ or $b \in M \setminus N$. This implication holds in $\context^{\delta}$ iff $A^{\delta} \subseteq \overline{b}^{\delta}$. In this paper all the implications with negative attributes are considered as implications of the dichotomized context.

 
An \emph{implication base} of a context $\context$ is defined as a set $\mathfrak{L}$ of implications of $\context$, from which any valid implication for $\context$ can be deduced by Armstrong rules and none of the proper subsets of $\mathfrak{L}$ has this property.  A cardinality minimal implication base was characterized in \cite{Duquenne1986} and is known as the \emph{canonical implication base}, or \emph{Duquenne-Guigues base}, or \emph{stembase}. In \cite{Ganter1984} the premises of implications of the canonical base were characterized in terms of pseudo-intents. A subset of attributes $P\subseteq M$ is called a \emph{pseudo-intent} if $P\not=P''$ and for every pseudo-intent $Q$ such that $Q\subset P$, one has $Q''\subset P$. The canonical implication base looks as follows: $\{ P \to (P''\setminus P) \mid P$ - pseudo-intent\}.
 
\section{Errors in Implicative Theories}\label{err_classes}
  
Without loss of generality we consider all observable properties to be expressed in terms of positive attributes from $M$. We aim at restoring valid implications and, therefore, correct errors in implicative theory of data. The goal is achieved if all implications are valid implications of the context. As already mentioned, all implications are reduced to unit ones.
 
Consider the following possible classes of implicative formulas ($A \subseteq M,\ b, c\in M$), which will be called dependencies:
\begin{enumerate}
    \item If an entity has all attributes from $A$, then it has attribute $b$ ($A \to b$); \label{1class}
    \item If an entity has all attributes from $A$, then it does not have attribute $b$ ($A \to \overline b$); \label{2class}
\end{enumerate}

\remark{In this work we consider only data domain dependencies in the form of implications with no negative attributes in the premise. It is possible to consider negative attributes in the premise by means of considering complementary context $(G,M, (G \times M) \setminus I)$. However, this is equivalent to introducing disjunction to our language: $A \to B \vee C\ :\Leftrightarrow\ A, \overline{B} \to C$. Then, having negation and disjunction we end up in the full propositional logic, for which computing the closure is not polynomial anymore. Therefore, it would not be possible to introduce a polynomial solution of this problem.}
 
Only formulas of Class 1 are standard FCA implications, formulas of Class 2 are FCA implications if the negation of attributes are explicitly introduced in the context. If there are no errors in a context, all the dependencies of Class \ref{1class} are deducible from an implication base. However, if not enough data is added to the context yet, we may get false consequences. Therefore, not all valid implications of the context have to necessarily be data domain dependencies. Nevertheless, it is guaranteed that none of valid dependencies is lost, and, as new objects are added, the number of false consequences is reduced (this is essentially the idea behind Attribute Exploration \cite{Ganter1999}). The situation is different if an erroneous object (data table row) is added. The erroneous object may violate a data domain dependency. In this case, until the error is found and corrected, we are not able to deduce all dependencies valid in the data domain from the implication base, no matter how many error-free objects are added afterwards.

\section{Finding Errors}\label{find}

We introduce two different approaches to finding errors. The first one is based on inspecting the canonical base of a context (base approach). When adding a new object to the context one may find all implications from the canonical base of the context such that the implications are not respected by the intent of the new object. These implications are then output as questions to an expert in form of implications. If at  least one of these implications is accepted, the object intent is erroneous. Since the canonical base is the most compact (in the number of implications) representation of all valid implications of a context, it is guaranteed that minimal number of questions is asked and no valid dependencies of Class \ref{1class} are left out. This approach can be seen as a version of pinpointing in the presence of only implicative rules.

Although this approach allows one to reveal all dependencies of Class \ref{1class}, there are several issues. The problem of producing the canonical basis with known algorithms is intractable. Recent theoretical results ~\cite{Kuznetsov2004On}, \cite{Distel2011a}, \cite{Kuznetsov2008Some}, \cite{Kuznetsov2013Comp} suggest that the canonical base can hardly be computed with better worst-case complexity than that of the existing approaches~\cite{Ganter1984}. One can use other bases (for example, there has been recent progress in computing proper premises \cite{RyDiB11}), but the algorithms known so far are still too costly and non-minimal bases do not guarantee that the expert is asked minimal sufficient number of questions.

However, since we are only interested in implications corresponding to one object at a time, it may be not necessary to compute the whole implication base. The second approach takes this fact into account. Let $A \subseteq M$ be the intent of the object under inspection; we separate it from the context. $m \in A''$ iff $\forall g \in G: A \subseteq g' \Rightarrow m \in g'$, in other words, $A''$ contains the attributes common to all object intents containing $A$. The set of unit implications $\{A \to b\ |\ b \in A''\setminus A\}$ can then be shown to the expert. If all implications are rejected, no attributes are forgotten in the new object intent. Otherwise, there are missing attributes in the object intent. Unfortunately, this simple observation does not allow to correct all the errors in implicative theory.

\begin{example}
Consider Error4 from Fig. \ref{err_cxt}. Error4 has set of attributes $A$ = \{has equal legs, has equal angles, all legs equal, at least 3 different legs\}. The closure $A''$ in the context from Fig. \ref{cxt} is equal to the set of all attributes $M$. Therefore, closure approach would ask if the user has forgotten to add all the attributes that are still possible to add. The suggestion to add all other attributes is not supported by any example in the context as there are no objects with all attributes. More than that, such solution is not minimal in general. Therefore, such a solution is not satisfactory.
\end{example}

A more general description of the situation in the example above is the following. Let $A \subseteq M$ be the intent of the inspected object such that $\nexists g \in G: A \subseteq g^{\prime}$. In this case $A^{\prime\prime} = M$ and the implication $A \to A'' \setminus A$ has an empty support. We could try to solve this problem by allowing to ask only those questions that have a supporting example in the context.

\begin{example}
Consider again Error4 from Fig. \ref{err_cxt}. As support for every question is required only the following question would be asked: has equal legs, has equal angles, at least 3 different legs $\to$ at least 3 different angles? Support: \{Quadrangle with 2 equal legs and 2 equal angles, Rectangular trapezium with 2 equal legs\}. However, a smaller and more intuitive correction would be to suggest the user to remove the attribute ``at least 3 different legs''. If this is indeed the source of error then even after adding the suggested attribute the error would not be eliminated and would impact the implicative theory.
\end{example}

At this point we conclude that it is necessary to be able to suggest corrections for errors of Class \ref{2class}. Such errors may be present if the object intent contains subset of attributes that none of the objects in the context has.

\subsection{Crucial Implications}

Suppose we have a new object $g_n$ with intent $A$ and we want to see whether $A$ respects (is consistent with) the previous knowledge given by the context, i.e., does not have errors of Classes 1 or 2. In order to find errors of Class 1 we need to know, whether, according to implications (implicative dependencies) of the context, the new object should have more attributes than just $A$. If this is the case, there should be an implication $B \to c$ not respected by $A$: $B \subseteq A$, but $c \not\in A$. Similarly, in order to find errors of Class 2, we look for implications $B \to \overline{c}$ such that $B \subseteq A$, but $c \in A$. The following proposition shows that we do not need to look for all such implications, but for a much smaller subset of them.

{\samepage
\begin{proposition}\label{proposition1}
Let $\context = \GMI,\ g_n' = A, A \subseteq M$. Let
$$\mathcal{I}_A^{(\context)} = \{B \to c\ |\ B\in \mathcal{MC}_A, c \in (B'' \setminus A) \cup (\overline{A \setminus B})\},$$
where $\mathcal{MC}_A = \{B \in \mathcal{C}_A\ |\ \nexists C \in \mathcal{C}_A: B \subset C\} \mbox{ and }\mathcal{C}_A = \{A \cap g'\ |\ g \in G\}$. The set $\mathcal{I}_A^{(\context)}$ contains (unit) implications with nonempty support that are valid in $\context$ and not respected by $A$. If an implication $(E \to d),\ E \subseteq A, d \in (M \setminus A) \cup \overline{A}$, with nonempty support is valid in $\context$, then there is an implication $(B \to d) \in {\cal I}_A^{(\context)}$ such that $E \subseteq B \subseteq A$.
\end{proposition}
}

The last statement says that the set of implications ${\cal I}_A^{(\context)}$ is enough to deduce every attribute that can be deduced from implications of the context with nonempty support. Implications from ${\cal I}_A^{(\context)}$ are called \emph{A-crucial} in $\context$. If ambiguity is excluded we omit the upper index and write simply ${\cal I}_A$.

\begin{proof}
Let $(B\to c) \in {\cal I}_A$, hence $B'\subseteq c'$ by the definition of implication. By definition of ${\cal I}_A$ one has $B = A\cap g'$ for some $g\in G$. Then $B\subseteq g'$ and by the antimonotonicity of $(\cdot)'$ one has $g''\subseteq B'$. Hence, $g''\subseteq B' \subseteq c'$ and $c''\subseteq g''' = g'$. Since $c\in c''$, one has $c\in g'$. Since $c\in g'$ and $B'\subseteq g$, by the properties of $(\cdot)'$, one has $(B\cup c)' = B'\cap c' \subseteq g$. Hence, the support of $B\to c$ contains $g$ and is not empty. Consider the following possible cases:
\begin{enumerate}
	\item $c \in B''\setminus A$. Since $B''\setminus A \not\subseteq A$ implication $B\to c$ is not respected by $A$;
	\item $c \in \overline{A\setminus B}$. Since $A\setminus B\subseteq A$, one has $c\not \in A$ and $B\to c$ is not respected by $A$.
\end{enumerate}

Now let $E \to d$ be a valid implication not respected by $A$ with a nonempty support. Then $E \subseteq A, d\notin A$ and there exists $g \in G$ such that $E\subseteq g', d \in g'$.
Therefore, there exists $B_{\mathcal{C}_A} \in \mathcal{C}_A$ such that $B_{\mathcal{C}_A} = A \cap g'$ and $E \subseteq B_{\mathcal{C}_A}$. Moreover, there exists $B_{\mathcal{MC}_A} \in \mathcal{MC}_A$ such that $B_{\mathcal{C}_A} \subseteq B_{\mathcal{MC}_A}$. By construction $B_{\mathcal{MC}_A} \subseteq A$, therefore, $E \subseteq B_{\mathcal{MC}_A} \subseteq A$. Consider the following possible cases:
\begin{enumerate}
	\item $d \in M$. As $E \subseteq B_{\mathcal{MC}_A}$ by the properties of $(\cdot)''$ one has $E'' \subseteq B_{\mathcal{MC}_A}''$. By definition of the validity of an implication one has $d \in E''$, hence, $d \in B_{\mathcal{MC}_A}''$. Therefore, $(B_{\mathcal{MC}_A} \to d) \in {\cal I}_A$;
	\item $d \in \overline{M}$. Let $\overline{c} = d$. For any $B \in \mathcal{MC}_A$ there exists $g_* \in G$ such that $B = A \cap g_*'$. If $E \subseteq g_*'$ then by the validity of the implication one has $d \in g_*'$, hence, $c \notin g_*'$. Therefore, $c \notin B$. As $d \notin A$ one has $c \in A$. Hence, $c \in A\setminus B$ and $d \in \overline{A\setminus B}$. Therefore, $(B_{\mathcal{MC}_A} \to d) \in {\cal I}_A$.\qed
\end{enumerate}
\end{proof}

\begin{proposition}\label{proposition2}
For a new object $g$ with intent $A$ one has ${\cal I}_A \leq |G|\times |M|$.
\end{proposition}
 
\begin{proof}
By definition of $\mathcal{MC}_A$ it contains no more than $|G|$ elements. For any $A, B \subseteq M$ one has $|B''| \leq |M|$, hence $|B'' \setminus A| \leq |M| - |A|$; $|\overline{A \setminus B}| \leq |A|$. Hence, $|(B'' \setminus A) \cup (\overline{A \setminus B})|\ \leq\ |(B'' \setminus A)| + |(\overline{A \setminus B})|\ \leq\ |M| - |A| + |A| = |M|$. Therefore, ${\cal I}_A$ contains not more than $|G|\times |M|$ implications.\qed
\end{proof}
 
According to Proposition \ref{proposition2} to check errors of Classes 1 and 2 one has to consider polynomialy many implications instead of exponentially many implications in the cardinality-minimum canonical base~\cite{Kuznetsov2004On}.

Proposition \ref{proposition1} allows one to design an algorithm for computing the set of questions (in form of implications) that can help to reveal possible errors of Classes \ref{1class} and \ref{2class}.

\begin{proposition}\label{proposition3}
Let $g$ be a new object with intent $A$. ${\cal I}_A$ can be computed in $O(|G|^2\times |M|)$ time.
\end{proposition}
 
\begin{proof}
  Consider the following \verb=inspect_closure= algorithm
  
  \begin{algorithm2e}[H]\label{pseudocode}
    \KwIn{$\context=(G, M, I),\ A \subseteq M$}
    \KwOut{$\mathcal{I}_A$}
    \If{$A''=A$}{\KwRet{$\emptyset$}}
    Candidates $= \{$object$'\cap A\ |\ $object $\in G\}$\\
    MaxCandidates $= \{C \in$ Candidates$\ |\ \nexists B \in $ Candidates$:\ C \subseteq B\}$\\
    Result $= \emptyset$\\
    \For{Candidate \textbf{in} MaxCandidates}
      {Result.add(\{Candidate $\to d\ |\ d \in ($Candidate$''\setminus A \cup\overline{A \setminus\text{Candidate}})\}$)}
  \KwRet{} Result
  \end{algorithm2e}
 
Here $A$ is the intent of the new object. In line 3 the algorithm computes the set of all subsets that are candidates for the premises of crucial implications. In line 4 all non-maximal subsets are discarded. In lines 6 and 7 closures of the premises are computed and the corresponding implications are added to the set of crucial implications. To estimate the worst-case complexity of the algorithm, note that executing line 1 and line 3 take at most O$(|G|\times |M|)$ time, line 4 takes O$(|M|)$ time for each of O$(|G|^2)$ containment tests, and lines 6 and 7 take O$(|G|\times |M|)$ time for computing closures of at most O$(|G|)$ premises of crucial implications. Hence, the total worst-case time complexity is O$(|G|^2\times |M|)$.\qed
\end{proof}

If there are several new objects, the set of crucial implications for each new object is in general dependent on the order of adding objects. The following statement shows which additional questions should be asked in order to compensate for this dependency.
 
\begin{proposition}\label{two_new_objects}
Let $g_1$ and $g_2$ be new objects with intents $A_1$ and $A_2$, respectively. Let $\context_1 = (G \cup g_1, M, I \cup \{(g_1, m)\ |\ m \in A_1\})$. If $\nexists g \in G: A_1 \cap A_2 \subseteq g'$ then $I_{A_2}^{(\context_1)} \setminus I_{A_2}^{(\context)} = \{A_2 \cap A_1 \to m\ |\ m \in (A_1 \setminus A_2) \cup (\overline{A_2 \setminus A_1})\}$, otherwise $I_{A_2}^{(\context_1)} = I_{A_2}^{(\context)}$.
\end{proposition}
 
\begin{proof}
By definition of $\cal{I}_A^{(\context)}$ only maximal intersections may become premises of implications. Hence, if there exists $g \in G$ such that $A_1 \cap A_2 \subseteq g'$ then no new implications can arise. However, if such $g$ does not exist, the set of new implications, by definition of $\cal{I}_A^{(\context)}$, is $\{(A_1 \cap A_2) \to m\ |\ m \in (A_1 \cap A_2)'' \setminus A_2 \cup \overline{A_2 \setminus (A_1 \cap A_2)}\}$. As $A_2 \setminus (A_1 \cap A_2) = A_2 \setminus A_1$ and, by assumption about maximal intersection, $(A_1 \cap A_2)'' = A_1$, the set of new implications is $\{A_2 \cap A_1 \to m\ |\ m \in (A_1 \setminus A_2) \cup (\overline{A_2 \setminus A_1})\}$. \qed
\end{proof}

Obviously, no more than $|(A_1 \setminus A_2) \cup (A_2 \setminus A_1)|$ additional questions may arise.
It is also easy to see that additional implications only arise in case where both new objects have maximal intents in $\context$. However, as we do not require any information about maximal intents in data domain, we have to be careful when adding an object with maximal intent. The following corollary states clearly which implications should be considered in order to guarantee the absence of errors that may affect the implication base.

\begin{proposition}\label{max_intent}
Let $g_n$ be a new object with intent $A$. The sets of implications $I_1 = \{A \to m\ |\ m \in M\setminus A,\,\nexists g \in G: A \subseteq g'\}$ and $I_2 = \{(A-a) \to \overline{a}\ |\ a \in A,\,\nexists g \in G: (A-a) \subseteq g'\}$ are valid in $\context$, have empty support, and are not respected by $A$.
\end{proposition}
 
Note that if there exists $g \in G$ such that $A \subseteq g'$ then both sets $I_1$ and $I_2$ are empty.
 
\begin{proof}
By definition $B \to c$ is respected by $N\subseteq M$ if $B\nsubseteq N$. By definition of $I_1$ and $I_2$ all the premises are not contained in any object intents from the context. Therefore, implications are valid, however, they are not supported by any of the object intents.
 
As $A\subseteq A,\, (M\setminus A) \nsubseteq A$, implications from $I_1$ are not respected by $A$. As $\forall a\in A: (A-a)\subseteq A,\, a \notin (M\setminus A)$, implications from $I_2$ are not respected by $A$. \qed
\end{proof}

According to Proposition \ref{max_intent} the number of additional questions for new objects that have maximal intents cannot exceed $|M|$. As none of the questions have objects from context in support we suggest that maximal objects should be checked ``by hand''.

For the sake of compactness in what follows we present implications in non-unit form. The name \verb=inspect_base= is used to denote the function implementing base approach.
 
\subsection{Example}
Consider the following example with convex quadrangles. Formal context given by the cross-table in Fig. \ref{cxt} contains convex quadrangles and their properties. The context does not cover the domain completely, i.e. not all possible convex quadrangle types are considered. Attributes ``has equal legs'' and ``has equal angles'' require at least two angles/legs of a quadrangle to be equal. Some dependencies on attributes are obvious, e.g., it is clear that if all angles are equal in a quadrangle, then this quadrangle definitely has equal angles.
 
Four objects are in the context of tentative errors in Fig \ref{err_cxt}. These objects are added to the context in Fig. \ref{cxt} one at a time.

\begin{figure}
    \centering
    \begin{cxt}
    \cxtName{Convex quadrangles}

    \atr{has equal legs}
    \atr{has equal angles}
    \atr{has right angle}
    \atr{all legs equal}
    \atr{all angles equal}
    \atr{at least 3 different angles}
    \atr{at least 3 different legs}

    \obj{xxxxx..}{Square}
    \obj{xxx.x..}{Rectangle}
    \obj{.....xx}{Quadrangle}
    \obj{xx.x...}{Rhombus}
    \obj{xx.....}{Parallelogram}
    \obj{.xx..xx}{Rectangular trapezium}
    \obj{x.x..xx}{Quadrangle with 2 equal legs and right angle}
    \obj{xx...x.}{Isosceles trapezium}
    \obj{xxx..xx}{Rectangular trapezium with 2 equal legs}
    \obj{.x...xx}{Quadrangle with 2 equal angles}
    \obj{x....xx}{Quadrangle with 2 equal legs}
    \obj{xx...xx}{Quadrangle with 2 equal legs and 2 equal angles}
    \end{cxt}
    \caption{Context of convex quadrangles $\context$}\label{cxt}
\end{figure}

\begin{figure}
    \centering
    \begin{cxt}
    \cxtName{Tentative errors}

    \atr{has equal legs}
    \atr{has equal angles}
    \atr{has right angle}
    \atr{all legs equal}
    \atr{all angles equal}
    \atr{at least 3 different angles}
    \atr{at least 3 different legs}

    \obj{x..x.x.}{Case1}
    \obj{x.xxx..}{Case2}
    \obj{.xxxxxx}{Case3}
    \obj{xx.x..x}{Case4}
    \end{cxt}
    \caption{Context of tentative errors $\context_{\mbox{e}}$}\label{err_cxt}
\end{figure}

{\samepage Inspecting Case1:
\begin{description}
    \item{\verb=inspect_base=}\hfill \\
        at least 3 different angles $\to$ at least 3 different legs \smallskip \\
        all legs equal $\to$ has equal angles, has equal legs
    \item{\verb=inspect_closure=}\hfill \\
        has equal legs, at least 3 different angles $\to$ at least 3 different legs,\\ \textoverline{all legs equal} \smallskip \\
        has equal legs, all legs equal $\to$ has equal angles, \textoverline{at least 3 different angles}
\end{description}}
Both algorithms reveal possible errors in a similar manner, although there are obvious differences. In the output of \verb=inspect_base= the premises are smaller than in the output of \verb=inspect_closure=. The latter also reveals dependencies of Class \ref{2class}. It is easy to see that all output implications hold in data domain. For example, if all legs are equal in a quadrangle, it should have equal angles and should not have 3 different angles. Hence, this object should be recognized as an error, it should be corrected to rhombus or to quadrangle with two equal legs.

\medskip

{\samepage Inspecting Case2:
\begin{description}
    \item{\verb=inspect_base=}\hfill \\
        all angles equal $\to$ has equal angles, has equal legs, has right angle \smallskip \\
        all legs equal $\to$ has equal angles, has equal legs
    \item{\verb=inspect_closure=}\hfill \\
        has right angle, has equal legs, all legs equal, all angles equal $\to$ has equal angles
\end{description}}
In this example we are able to ask even less questions to an expert using \verb=inspect_closure= as with \verb=inspect_base=. This is the result of finding implications generated by maximal subsets of object's intent. The intent of Case2 occurs in the context (in the intent of Square), that is why we do not get any negative attributes in the output of \verb=inspect_closure=. Again, all implications are valid in data domain, therefore, Case2 is an error, it should be corrected to square.

\medskip

{\samepage Inspecting Case3:
\begin{description}
    \item{\verb=inspect_base=}\hfill \\
        all angles equal $\to$ has equal angles, has equal legs, has right angle\\ \smallskip
        all legs equal $\to$ has equal angles, has equal legs
    \item{\verb=inspect_closure=}\hfill \\
        has equal angles, has right angle, at least 3 different legs, at least 3 different angles $\to$ \textoverline{all angles equal}, \textoverline{all legs equal} \smallskip \\
        has equal angles, has right angle, all legs equal, all angles equal $\to$ has equal legs, \textoverline{at least 3 different angles}, \textoverline{at least 3 different legs}
\end{description}}
In Case3 we get both implications from the output of \verb=inspect_base= combined in one implication with a bigger premise in the output of \verb=inspect_closure=. In addition we obtain several implications with negative attributes. It is easy to see that all implications hold in the data domain, therefore, Case3 is an error and should be corrected either to rectangular trapezium or to square.

\medskip

{\samepage Inspecting Case4:
\begin{description}
    \item{\verb=inspect_base=}\hfill \\
        has equal angles, has equal legs, at least 3 different legs, all legs equal $\to$ has right angle, at least 3 different angles, all angles equal
    \item{\verb=inspect_closure=}\hfill \\
        has equal angles, has equal legs, all legs equal $\to$ \textoverline{at least 3 different legs}\smallskip \\
        has equal angles, has equal legs, at least 3 different legs $\to$ at least 3 different angles, \textoverline{all legs equal}
\end{description}}
Case4 is a very special case where the corresponding implication from canonical base has empty support. In the output of \verb=inspect_base= we obtain all  questions possible for this intent. As discussed above these questions are not based on any information input so far. The reason for that is that Case4 has maximal intent in the context. So these questions could also be found using Proposition \ref{max_intent}. However, even if we add attributes ``at least 3 different angles'' and ``all angles equal'' and reject the last implication we would not be able to recognize this object as an error. On the contrary \verb=inspect_closure= allows us to recognize errors of Class \ref{2class} and state that Case4 should be corrected to have the intent of rhombus or quadrangle with two equal legs and two equal angles.

\section{Experiment}\label{experiment}

{Below the results of experiments on synthetic data are presented. The experiments were conducted as follows: all objects are first taken one by one out from the context and then added as new objects; all the possible errors of Classes 1 and 2 are found and output. This experiment is run for the purpose of testing the efficiency of the algorithm, \emph{not} as an attempt to find errors.

An FCA package for Python was used for implementation (\cite{RevenkoPyFCA}). For computing the canonical base an optimized algorithm based on \verb=Next Closure= was used (\cite{Obiedkov2007}). All tests described below were run on computer with Intel Core i7 1.6GHz processor and 4 Gb of RAM running Linux Ubuntu 11.10 x64.}

In Fig. \ref{graph} the results of running both algorithms on synthetic contexts are presented. For each context the number of objects is equal to 50. Parameter $d$ represents the density of the context, i.e. the probability of having a cross in the cross-table representing the relation. This result is presented in the semi-logarithmic scale. It is easy to note that with the growth of  the number of attributes and the density, the difference between runtime of two algorithms grows as well.

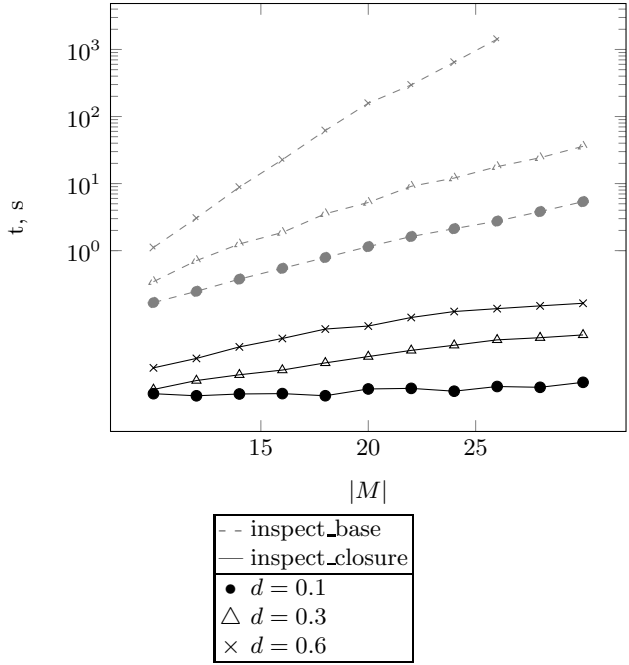
\begin{figure}
\centering
\begin{tikzpicture}
    \begin{axis}[xlabel=$|M|$, ylabel={t, s},
                  legend columns=2,
                  ymin=0, ytick={1,10,100, 1000}, ymode=log,
                  xtick={15,20,25}]
            \addplot[mark=*,gray, dashed] coordinates {
                (10, 0.167965544595)
                (12, 0.247017635239)
                (14, 0.375766846869)
                (16, 0.543842315674)
                (18, 0.789092620214)
                (20, 1.14722710186)
                (22, 1.62289366457)
                (24, 2.12516090605)
                (26, 2.75543173154)
                (28, 3.8202943537)
                (30, 5.37453644805)
                };
                  
        	   \addplot[mark=*,black] coordinates {
                (10, 0.0073512395223)
(12, 0.00683691766527)
(14, 0.00727897220188)
(16, 0.00737159781986)
(18, 0.00684918297662)
(20, 0.008645468288)
(22, 0.0088245206409)
(24, 0.008004811075)
(26, 0.00940683152941)
(28, 0.00915070374807)
(30, 0.010860840479)
                };
                
            \addplot[mark=triangle,gray, dashed] coordinates {
                 (10, 0.344418048859)
                 (12, 0.709496061007)
                 (14, 1.25616751777)
                 (16, 1.86424203714)
                 (18, 3.54632928636)
                 (20, 5.22309121821)
                 (22, 9.10261064106)
                 (24, 12.0801133447)
                 (26, 17.8718641599)
                 (28, 24.3004840745)
                 (30, 35.9667735497)
                };                

            \addplot[mark=triangle,black] coordinates {
                (10, 0.00849596659342)
                 (12, 0.0115488237805)
                 (14, 0.0140699677997)
                 (16, 0.016548289193)
                 (18, 0.0211648543676)
                 (20, 0.0263027217653)
                 (22, 0.0325202941895)
                 (24, 0.0387960539924)
                 (26, 0.0466770728429)
                 (28, 0.0502827829785)
                 (30, 0.0553248738183)
                };
                
            \addplot[mark=x,gray, dashed] coordinates {
(10, 1.11483882533)
(12, 3.05852622456)
(14, 8.86141722732)
(16, 22.6997146871)
(18, 61.8963118394)
(20, 157.937982559)
(22, 295.927075134)
(24, 646.7)
(26, 1424)
                };                

            \addplot[mark=x,black] coordinates {
(10, 0.0178514321645)
(12, 0.0246753030353)
(14, 0.0366503265169)
(16, 0.0489775472217)
(18, 0.0676475498411)
(20, 0.0750198496713)
(22, 0.10060983234)
(24, 0.1235)
(26, 0.1364)
(28, 0.15)
(30, 0.164)
                };
        \end{axis}
\end{tikzpicture}

\centering
\begin{tabular}{|c l|}
\hline
\textcolor{gray}{- -} & inspect\_base\\
\textcolor{black}{---} & inspect\_closure\\
\hline
$\bullet$ & $d=0.1$\\
$\triangle$ & $d=0.3$\\
$\times$ & $d=0.6$\\
\hline
\end{tabular}
\caption{Comparison of runtime on synthetic contexts in semilog scale.}\label{graph}
\end{figure}

Another experiment was conducted to test the quality of finding errors by the introduced method. The information about dependencies between negative attributes is not reflected in the implication base. Therefore, more implications are usually violated by objects having more attributes in their intent. Small intents usually violate only few implications. However, in this experiment we aim at finding not only the errors affecting the implication base; therefore, it is necessary to level out the shift between larger and smaller intents. For this purpose in the following experiments a slight modification of the introduced method is used. The complementary context to a context $\context = \GMI$ is defined as $\context^c := (G, M, (G\times M) \setminus I)$. The method applied to the complementary context will output implications with only negative attributes in the premise. Running the introduced method on both original context and complementary context yields better results. Note that implications with both positive and negative attributes will not be generated.

The experiments were conducted in the following settings. An object was picked up from a context, from one to three errors were randomly introduced into the intent of the object. The method was used to find possible errors in the object. If all the erroneous attributes were in the conclusion of the unit implications with \emph{the same} premise then the errors were marked as found. In this case all the erroneous attributes are contained in one question to the user. Afterwards the object already without errors is returned to the context and the next object is picked up. We considered three contexts from the UCI repository (\cite{Frank2010}): \texttt{SPECT}, \texttt{house-votes-84} and \texttt{kr-vs-kp}. Therefore, nine experiments were conducted. In every experiment 1000 objects (with possible repetitions) were picked up one after another. In Table \ref{table:quality} the results of the experiment are presented.

\begin{table}
    \begin{subtable}{\textwidth}
        \centering
        \begin{tabular}{| c | c | c | c |}
	\hline
	Context Name & $|G|$ & $|G|$ after reducing & $|M|$ \\
	\hline
	\verb=SPECT= & 267 & 133 & 23 \\
	\verb=house-votes-84= & 232\footnotemark[1] & 104 & 17 \\
	\verb=kr-vs-kp= & 3198 & 453 & 36\footnotemark[2] \\
	\hline
	\end{tabular}
    \end{subtable}
    \bigskip\\

    \begin{subtable}{\textwidth}
        \centering
        \begin{tabular}{| c | c | c | c | c | c |}
	\hline
	\parbox[][1.6cm][c]{1.3cm}{\centering Number of Errors per Object} & \parbox{1.3cm}{\centering Errors Found} & \parbox{1.8cm}{\centering Found / All Ratio} & \parbox{1.8cm}{\centering Total Number of Implications} & \parbox{1.8cm}{\centering Implications per Object} \\
	\hline
	1 & 548 & 0.548 & 2298 & 2.41 \\
	2 & 242 & 0.242 & 2753 & 2.80 \\
	3 & 131 & 0.131 & 2703 & 2.71 \\
	\hline
	\end{tabular}
        \caption{\texttt{SPECT}}
    \end{subtable}

    \begin{subtable}{\textwidth}
        \centering
        \begin{tabular}{| c | c | c | c | c | c |}
	\hline
	\parbox[][1.6cm][c]{1.3cm}{\centering Number of Errors per Object} & \parbox{1.3cm}{\centering Errors Found} & \parbox{1.8cm}{\centering Found / All Ratio} & \parbox{1.8cm}{\centering Total Number of Implications} & \parbox{1.8cm}{\centering Implications per Object} \\
	\hline
	1 & 712 & 0.712 & 9780 & 11.4 \\
	2 & 217 & 0.217 & 14018 & 14.7 \\
	3 & 71 & 0.071 & 18276 & 18.4 \\
	\hline
	\end{tabular}
        \caption{\texttt{house-votes-84}}
    \end{subtable}

    \begin{subtable}{\textwidth}
        \centering
        \begin{tabular}{| c | c | c | c | c | c |}
	\hline
	\parbox[][1.6cm][c]{1.3cm}{\centering Number of Errors per Object} & \parbox{1.3cm}{\centering Errors Found} & \parbox{1.8cm}{\centering Found / All Ratio} & \parbox{1.8cm}{\centering Total Number of Implications} & \parbox{1.8cm}{\centering Implications per Object} \\
	\hline
	1 & 786 & 0.786 & 7520 & 8.47 \\
	2 & 393 & 0.393 & 12863 & 13.2 \\
	3 & 247 & 0.247 & 18322 & 18.3 \\
	\hline
	\end{tabular}
        \caption{\texttt{kr-vs-kp}}
    \end{subtable}
    \caption{Error finding experiment carried out on three contexts from UCI.}
    \label{table:quality}
    \bigskip
    \underline{\hspace{2cm}}

    {\footnotesize
    \hskip 0.5cm ${}^1$All objects containing missing values were removed.

    \hskip 0.5cm ${}^2$Attribute 15 was removed due to many-valuedness.}
\end{table}

As can be seen from the results the more objects there are in the context, the better method works. In \texttt{SPECT} intents are very diverse (there are only 5.78 irreducible objects per attribute on average) that is why not more than three implications on average are output and bad ratio of found errors is obtained. In \texttt{house-votes-84} intents are more similar, that is why we have more questions per object. The ratio of found errors for one error is relatively high, however, it quickly drops with the increase of errors, as the number of irreducible objects is small. In \texttt{kr-vs-kp} there are much more objects per attribute and the results for two and three errors at a time are much better. However, if the user is able to correctly answer all unit implications even better results can be achieved. In this case the user may correct first errors and repeat the procedure having already only one error. In these experiments the error-finding process was considered successful only if there is one implication suggesting all the needed corrections at once.

It is worth noting that  the chance of random guess in predicting all errors in an object description is only $1/(|M|^n)$ if there are $n$ errors as compared to 50\% for the classification task.

\section{Conclusion}
A method for finding errors  in implicative theories was introduced. The method uses some techniques based on Formal Concept Analysis. As opposed to finding the canonical (cardinality minimal) base of implications, which can be very time consuming due to intrinsic intractability, the proposed algorithm terminates in polynomial time. Moreover, after checking maximal object descriptions (object intents) ``by hand'' it is possible to find all errors of two considered types or prove their absence. Computer experiments show that in practice the proposed method works much faster than that based on the generation of the implication base.

\subsubsection{Acknowledgements}
The first author was supported by German Academic Exchange Service (DAAD). The second author was supported by the Basic Research Program of the National Research University Higher School of Economics, project Mathematical models, algorithms, and software for knowledge discovery in structured and text data. We thank Bernhard Ganter and Sergei Obiedkov for discussion and useful remarks.

\bibliographystyle{plain}
\bibliography{mathrefs,my_refs}

\end{document}